%% file: collas2022_conference.tex
\documentclass{article} 
\usepackage{collas2022_conference,times}

\input{math_commands.tex}

\usepackage{hyperref}
\hypersetup{
    colorlinks=true,
    linkcolor=red,
    filecolor=magenta,      
    urlcolor=blue,
    citecolor=purple,
    pdftitle={Overleaf Example},
    pdfpagemode=FullScreen,
    }

\usepackage[noend]{algorithm,algorithmic}
\usepackage{caption}
\usepackage{amsmath, amssymb}
\usepackage[utf8x]{inputenc} 
\usepackage{hyperref}
\usepackage{interval}
\usepackage{caption}
\usepackage{hyperref}
\usepackage{cleveref}%
\usepackage{amsthm}

\usepackage{tikz}
\setcitestyle{square}
\setcitestyle{numbers}

\definecolor{darkgray}{rgb}{0.50, 0.50, 0.50}
\definecolor{gray}{rgb}{0.70, 0.70, 0.70}
\definecolor{lightgray}{rgb}{0.92, 0.92, 0.92}

\newtheorem{theorem}{Theorem}
\newtheorem{corollary}[theorem]{Corollary}

\title{Matrix Completion with Heterogeneous Cost}

\author{Ilqar Ramazanli \\
Carnegie Mellon University\\
Pittsburgh, USA \\
\texttt{iramazan@alumni.cmu.edu}  
}

 \collasfinalcopy 

\begin{document}

\maketitle

\begin{abstract}
The matrix completion problem has been studied broadly under many underlying conditions.
In many real-life scenarios, we could expect elements from distinct columns or distinct positions to have a different cost.
In this paper, we explore this generalization under adaptive conditions.
We approach the problem under two different cost models. 
The first one is that entries from different columns have different observation costs, but, within the same column, each entry has a uniform cost.
The second one is any two entry has different observation cost, despite being the same or different columns.
We provide complexity analysis of our algorithms and provide tightness guarantees.
\end{abstract}

\vspace{7mm}

\section{Introduction}

Matrix completion and compressed sensing has gained attention of machine learning and signal processing communities recently.
This attention particularly increased since modern data analysis gained research attention and low-rank structures has started to be studied in detail.
Announcement of Netflix Prize problem has especially triggered these studies and many researchers started studying underlying structure of real-life datasets and their singular value decomposition.
One of the earliest work in this field is due to \cite{recht1} which it discusses that the problem could be solved as convex relaxation of the simplest problem that could be attained from observed data:

\begin{align*}
    \text{minimize} \hspace{5mm}  & \text{rank}( X ) \\
    \text{subject to} \hspace{5mm} & X_{ij} = M_{i,j} \text{ for } (i,j) \in \Omega
\end{align*}

which it has been discussed in \cite{chistov1984complexity} that solving this problem is exponential in terms of the size of the matrix both theoretically and practically. 
Therefore, the following relaxation of the problem has been suggested in \cite{recht1}:

\begin{align*}
    \text{minimize} \hspace{5mm}  & \| X \|_* \\
    \text{subject to} \hspace{5mm} & X_{ij} = M_{i,j} \text{ for } (i,j) \in \Omega
\end{align*}

where $\| . \|_*$ stands for nuclear norm minimization.
Authors has proved that given $| \Omega | > C n^{1.2} r \log{n}$ and if the matrix is not coherent then the problem above has a unique solution given $n = \mathrm{max}(n_1, n_2)$.
In follow-up works, the problem has been studied in more details and more tight bounds has been suggested and proved \cite{gross, recht2, tao}.\\[2ex]
Recently it has been shown in many problems adaptive sensing gives stronger performance guarantees \cite{paramythis2003adaptive, riedmiller1992rprop} compared to traditional methods. 
Active learning has been comprehensively studied in many survey papers   \cite{settles2009active, arnold2005survey}.
For optimization problems adaptivity has been studied in \cite{zeiler2012adadelta, becigneul2018riemannian, ramazanli2020adaptive, ma2019adequacy}, for curriculum learning it has been studied in \cite{kong2021adaptive, huang2020curricularface, zhao2020egdcl} and the idea has been applied to many other problems.

In terms of compressed sensing and exact matrix completion some of the earliest works are due \cite{frieze2004fast, deshpande2006matrix, akshay1}. 
Later many active matrix completion algorithms proposed by researchers \cite{ruchansky2015matrix, misir2019active, bhargava2017active}.
In the streamline of works \cite{akshay1, akshay2} authors suggests a single phase algorithm which later on further extended by \cite{nina, poczos2020optimal, ramazanli2022adaptive, ilqarsingle}.\\[2.5ex]
Matrix completion problem has been studied both with noise and noiseless setting. 
In noisy setting, generally we try to estimate the underlying low-rank structure, using the idea of to threshold singular values.
Some of the example works are \cite{candes2010matrix, keshavan2009matrix,cai2010singular,rohde2011estimation, xie2017low}.
In noiseless setting, we generally call the problem exact low-rank matrix completion and it also has been in the center of attention of many researchers \cite{jain2015fast, zhang2011strongly, candes2008exact, bertsimas2020fast}. \\[2.5ex]
So far all the discussed matrix completion algorithms described here works assuming observation complexity of each entry is the same across entries.
An interesting and intuitive extension of this problem is to assume different entries of the matrix having different cost.
It is in a way intuitive, that in an example problem of rating movies, watching longer movies requires more time cost than watching shorter ones.
Also assuming that different people having different time availability, even the cost of the same length movie could be perceived to be different across various people.\\[2.5ex]
Firstly, we will assume that underlying $m\times n$ sized rank-$r$ matrix $\mathbf{M}$ has different cost for each column. 
Specifically, there is a set of costs $\{c_1, c_2, \ldots, c_n\}$ which any entry $\mathbf{M}_{ij}$ has the cost $c_j$ no matter which index we pick for $i$.
Secondly, we assume that there is a set of cost entries $\{c_{11}, \ldots c_{mn}\}$ that each entry $i,j$ of the matrix $\mathbf{M}_{ij}$ has the cost of $c_{ij}$.
We give different optimality condition based on the different cost model.

\vspace{5mm}
\subsection{Preliminaries}

We briefly introduce notations and definitions used in this paper. 
We use similar notations as it has been used before in the literature to keep the consistency.
$\mathbf{M}$ stands for the target matrix that we want to recover which has size of $m \times n$ and rank of $r$.  
$\| x\|$ denote the $L_{2}$ norm of a vector $x \in \mathbb{R}^n$.
$x_{\Omega}$ denote the sub-vector of $x$ where coordinates are from $\Omega$.
For any $\mathbf{R}\subset[m]$, $\mathbf{M}_{\mathbf{R}:}$ stands for an $|\mathbf{R}| \times n$ sized submatrix of $\mathbf{M}$ that rows are restricted by $\mathbf{R}$. 
We define $\mathbf{M}_{:\mathbf{C}}$ in a similar way for restriction with respect to columns.
Intuitively, $\mathbf{M}_{\mathbf{R}:\mathbf{C}}$ defined for $|\mathbf{R}|\times |\mathbf{C}|$ sized submatrix of $\mathbf{M}$ with rows restricted to $\mathbf{R}$ and columns restriced to $\mathbf{C}$.
Moreover, for the special case $\mathbf{M}_{i:}$ stands for $i$-th row and $\mathbf{M}_{:j}$ stands for the $j$'th column.
Similarly, $\mathbf{M}_{i:\mathbf{C}}$ will represent the restriction of the row $i$ by $C$ and $\mathbf{M}_{\mathbf{R}:j}$ represents restriction of the column $j$ by $\mathbf{R}$.
As it's mentioned in \cite{ramazanliadaptive} the sparsity number is defined by first defining the nonsparsity number:

\begin{align*}
    \psi(\mathbf{M})=\mathrm{min}\{\|x\|_0 | x=\mathbf{M}z \text{ and } z\notin \mathrm{null}(\mathbf{M})  \} \hspace{15mm}  \psi(\mathbb{U})=\mathrm{min}\{\|x\|_0 | x\in \mathbb{U} \text{ and } x\neq 0 \}
\end{align*}

then using this number we define sparsity-number:

\begin{align*}
\overline{\psi}(x) = m-\psi(x) \hspace{20mm}
\overline{\psi}(\mathbf{M})=  m- \psi(\mathbf{M}) \hspace{20mm}
\overline{\psi}(\mathbb{U})=  m- \psi(\mathbb{U}) 
\end{align*}

\vspace{5mm}

\section{Main Results}

The algorithm $\mathbf{ERCS}$ is just modification of the algorithm \cite{akshay1} or \cite{ilqarsingle} in the input phase. 
We show observing just $d=\overline{\psi}(\mathbb{U})+1$ many entries in each column is enough te decide whether partially observed column is contained in the subspace or not.\\[2ex]
Under the condition that the column space $\mathbb{U}$ of the underlying matrix satisfies $r-1=\overline{\psi}(\mathbb{U})$, the observation complexity of the algorithm $\mathbf{ERCS}$ becomes $m\times r +(n-r)\times r = (m+n-r) r$ which is the degree of freedom of the set of $m\times n$ sized rank-$r$ matrices.
Therefore, under this condition $\mathbf{ERCS}$ is absolutely optimal

\begin{algorithm}
\caption*{ \hypertarget{ercs}{\textbf{ERCS: }}Exact recovery with column sparsity}
 \textbf{Input:}   $d=\overline{\psi}(\mathbb{U})+1$

 \textbf{Initialize:}  $k=0 , \widehat{\mathbf{U}}^0 = \emptyset$ 

\begin{algorithmic}[1]
    \STATE Draw uniformly random entries $\Omega \subset [m]$ of size $d$   
    \STATE Observe entire $\mathbf{M}_{\Omega:}$ 
    \FOR{$i$ from $1$ to $n$}        
    \STATE  \hspace{0.2in} \textbf{if} $\| \mathbf{M}_{\Omega:i}-\mathcal{P_U}_{\Omega} \mathbf{M}_{\Omega:i}\| >0$ 
    \STATE \hspace{0.2in}  \hspace{0.2in} Fully observe $\mathbf{M}_{:i}$ 
    \STATE  \hspace{0.2in}  \hspace{0.2in}  $\widehat{\mathbf{U}}^{k+1} \leftarrow \widehat{\mathbf{U}}^{k} \cup \mathbf{M_{:i}} $
    \STATE \hspace{0.2in}  \hspace{0.2in} Orthogonalize  $\widehat{\mathbf{U}}^{k+1}$
    \STATE \hspace{0.2in} \hspace{0.2in} $k=k+1$          
   \STATE   \hspace{0.2in}  \textbf{otherwise:} $\widehat{\mathbf{M}}_{:i} = \widehat{\mathbf{U}}^k \widehat{\mathbf{U}}^{k^+}_{\Omega:} \widehat{\mathbf{M}}_{\Omega :i}$
\ENDFOR

\end{algorithmic}
 \textbf{Output:} Underlying matrix $\widehat{\mathbf{M}}$
\end{algorithm}

\begin{theorem}
Let $\mathbb{U}$ represent the column space of the $m\times n$  sized matrix $\mathbf{M}$ of rank $r$. 
Then, $\mathbf{ERCS}$ exactly recovers $\mathbf{M}$ by $$m\times r  + (n-r)(\overline{\psi}(\mathbb{U})+1)$$
observations.
\end{theorem}
\begin{proof}
We start by showing $\mathbf{ERCS}$ recovers $\mathbf{M}$ exactly and later we focus on observation count.
To prove correctness of exact recovery we use mathematical induction as follow:\\[2.5ex]
$Hypothesis$ : after $i$-th iteration $\mathbf{ERCS}$ already  correctly recovered first $i$ columns.\\[2.5ex]
$Base$ $case$ : $i=1$ is trivial as if at least one of the observed entries is nonzero we completely observe the column, which guarantees correctness. 
On the other hand, if it happens all of $\overline{\psi}(\mathbb{U})+1$ entries are zero, then the first column is indeed completely zero because the definition of the \textit{space sparsity number} implies there can be at most $\psi(\mathbb{U})$ many zero coordinates in a nonzero vector in the column space.\\[2.5ex]
$Hypothesis$ $proof$ : 
Let assume after step $i-1$, $\mathbf{ERCS}$ recovered first $i-1$ columns correctly and we want to show the algorithm exactly recovers $i$-th column too. We use the lemma 1 and 6 from \cite{poczos2020optimal} to verify this fact:\\[2.5ex] 
\textbf{Lemma 1:} \label{lm1} Let $\mathbb{U}$ be a subspace of $\mathbb{R}^{m}$ and $x^1,x^2,...,x^n$ be any set of vectors from $\mathbb{U}$. Then the linear 
dependence of $x^1_\Omega, x^2_\Omega,...,x^n_\Omega$ implies linear dependence of 
$x^1,x^2,...,x^n$, for any $\Omega \subset [m]$ such that $|\Omega| > \overline{\psi}(\mathbb{U})$.\\[2.5ex]
\textbf{Lemma 6:} 
Let $\mathbb{U}$ be a subspace of $\mathbb{R}^{m}$ and $x^1,x^2,...,x^n$ be any set of vectors from $\mathbb{U}$. Then the linear 
dependence of ${x^1}, {x^2},...,{x^n}$ implies linear dependence of 
$x^1_\Omega, x^2_\Omega,...,x^n_\Omega$, for any $\Omega \subset [m]$.
\\[2.5ex]
From the design of the algorithm $\mathbf{M}_{\Omega : {i}}$ is already observed.
Then, if in the line 4, $\mathbf{ERCS}$ decides the column is linearly independent with previous columns, as in the next line we completely observe the column there is no chance that the algorithm can do mistake under this case.
Therefore, the only remaining case is, if in the line 4 the algorithm decides the column $i$ is linearly dependent.\\[2.5ex]
From the statement of lemma 1, if a set of vectors from a subspace $\mathbb{U}$ are linearly dependent on a given subset of coordinates, then they are indeed linearly dependent. 
We conclude that the algorithm's decision is correct and by just back projection method, the algorithm recovers remaining entries of the partially observed column.
Therefore, column $i$ also recovered correctly and we are done with the proof of induction hypothesis. \\[2.5ex]
Our next goal is to show the observation complexity is $m\times r  + (n-r)(\overline{\psi}(\mathbb{U})+1)$.
From the lemma 1, we conclude that whenever current column is indeed linearly independent with previous columns, the $\mathbf{ERCS}$ also decides it is linearly independent.
Moreover, from the lemma 10, we conclude that if the current column is linearly dependent with previous columns, then in this case $\mathbf{ERCS}$ decides it is linearly dependent.
As there are $r$ many linearly independent columns in the underlying matrix $\mathbf{M}$, the algorithm decides independence exactly $r$ times and in each of them it does complete observations.
However, in remaining $n-r$ columns, number of observations is exactly $\overline{\psi}(\mathbb{U})+1$.
As a conclusion, number of total observations is: $rm+(n-r)(\overline{\psi}(\mathbb{U})+1)$
\end{proof}
\vspace{3mm}
\begin{corollary}
$\mathbf{ERCS}$ still performs correctly under the case in each column, number of observed entries is strcitly lower bounded by $\overline{\psi}(\mathbb{U})$ : $d \geq \overline{\psi}(\mathbb{U})+1$. Moreover, the number of observations will be $m\times r + (n-r)d$ 
\end{corollary}
\begin{proof}
The proof is exactly proceeds as proof of the theorem.
The key point is to notice, lemma 1 and lemma 6 are still satisfying.
\end{proof}
\vspace{2mm}
\paragraph{Optimality of $\mathbf{ERCS}$:}
We notice that many adaptive algorithms such as \cite{akshay1} and $\mathbf{ERCS}$ has two stages of observations.

\begin{itemize}
    \item[$i:$] Select subset of rows and observe them completely.
    \item[$ii:$] Detect linearly independent columns and observe them completely.
\end{itemize}
The discussion for tightness of the lemma 1 implies that $\mathbf{ERCS}$ is optimal deterministic two stage observation, low-rank exact recovery algorithm.
Moreover, in the corollary we discussed for any $d\geq \overline{\psi}(\mathbb{U})$ 
the $\mathbf{ERCS}$ algorithm would still perform correctly.
Therefore, having constant factor approximation of the \textit{space sparsity number} of the column space would lead asymptotically optimal algorithm:

\begin{align*}
  \widehat{d} \leq K d & \implies \widehat{d}+1 \leq K(d+1)     \\
  & \implies   (n-r)(\widehat{d}+1) \leq (n-r)K(d+1)            
\end{align*}

adding $mr$ to both side leads to

\begin{align*}
      rm+ (n-r)(\widehat{d}+1) &\leq rm +(n-r)K(d+1)            \\
                               & \leq Krm + K(n-r)(d+1)         \\
                               & = K\big(rm+(n-r)(d+1)\big)
\end{align*}

Notice that $rm+ (n-r)(\widehat{d}+1)$ is the observation complexity we have and $\big(rm+(n-r)(d+1)\big)$ is optimal two stage as we discussed here.
All together, the inequality  implies constant approximation to \textit{space sparsity number} gives as constant approximation to optimal solution.\\[2.5ex]
Here, we discuss matrix completion problem with relatively different setting.
Previously, we focused on the case, where each of the entry of underlying matrix $\mathbf{M}$ has the same observation cost.
In this section we discuss the completion problem where entries of the matrix has non uniform cost to observe.
We tackle with two type of non uniformity here: 
\begin{itemize}
    \item Entries has uniform cost across the same column, but different columns has different costs. 
    \item Each entry of the matrix has different cost.
\end{itemize}

\vspace{4mm}
\subsection{Uniform Cost Across Columns}

\textbf{Problem:} For any fixed $j$, the cost of observing $\mathbf{M}_{i:j}$ is equal to $\chi_j$ for any $1\leq i \leq m$, and $\chi_1, \chi_2,\ldots, \chi_n$ are arbitrary positive numbers and we target to recover the matrix $\mathbf{M}$ as cheap as possible. \\[2.5ex]
\textbf{Solution:} We propose a slight modification of the $\mathbf{ERCS}$ to solve optimally among the two staged methods as we discussed before.
Lets remind that in the algorithm we show that selecting any $d=\overline{\psi}(\mathbb{U})+1$ many rows is enough to guarantee exact recovery deterministically.
In the next stage, we iteratively go through columns one by one starting with the first column, and if we detect a column is linearly independent with previous ones, we completely observe it.
If not, we recover it using the pre-determined subspace.\\[2.5ex]
To adapt the solution for this problem, we just need to change the order of the columns we start to check. 
Basically, instead of starting with the first column, we should start with the cheapest one.
If we decide its not contained in the current subspace, we completely observe all entries and if it is contained then we just recover with the current subspace.
Then, we move to second cheapest column and so on so forth with the increasing order of cost.\\[2.5ex]
\textit{Correctness:} We can see that the proof of the correctness of $\mathbf{ERCS}$ is independent of the order of the columns. 
Therefore, selecting columns with increasing order of the cost would not change the correctness of the algorithm. \\[2.5ex]
\textit{Optimality:} The set of two stage algorithm can be parametrized by two numbers. 
First one is - $d$- the number of rows fully observed and the second is the subset of indices of columns to observe fully. 
We analyse the optimal algorithm for three cases of values of $d$:\\\\
1. $d \leq \overline{\psi}(\mathbb{U})$. 
It is obvious that optimal algorithm cannot have $d \leq \overline{\psi}(\mathbb{U})$, because from the discussion for tightness of the lemma 1 and optimality of $\mathbf{ERCS}$, there are matrices that selection of $d = \overline{\psi}(\mathbb{U})$ rows is not enough to guarantee the existence of $r$ linearly independent rows.
\medskip\\
2. $d = \overline{\psi}(\mathbb{U})+1$ It is a well known fact that the set of column basises are matroids and Greedy algorithms gives the optimal solution for matroids \cite{helman1993exact}.
Note that the algorithm discussed is efficient way of giving greedy solution.
\medskip\\
3. $d > \overline{\psi}(\mathbb{U})+1$. Lets assume that the optimal algorithm takes $\Tilde{d}>d$ rows in the first phases and  columns: $i_1,i_2,\ldots, i_{\Tilde{r}}$. 
We first note that, $\Tilde{r}=r$, it is because if $\Tilde{r}<r$ then selected columns are not enough to learn the column space and if $\Tilde{r}>r$  we can pick subset of these columns that is basis for column space and selecting this basis has less cost which contradicts to optimality. Therefore, $\Tilde{r}=r$ for optimal case.
Moreover, we can use the same subset selection argument to pick $\overline{\psi}(\mathbb{U})+1$ sized subset of rows then select the same set of columns and it will be cheaper.
Therefore, for optimality we should select exactly $d= \overline{\psi}(\mathbb{U})+1$ rows.

\vspace{5mm}
\subsection{Exact recovery with full heterogeneity }
\textbf{Problem:} For any $i,j$, the cost of observing $\mathbf{M}_{i:j}$ is equal to $\chi_{ij}$ and $\chi_{11}, \chi_{12},\ldots, \chi_{mn}$ are arbitrary positive numbers and similar to the previous problem we target to recover the matrix $\mathbf{M}$ as cheap as possible. 
\medskip\\
\textbf{Solution:} We describe the solution in the following algorithm:
\begin{algorithm}
\caption*{\textbf{ERHC: } \hypertarget{erhc}{Exact} recovery with heterogeneous cost}
 \textbf{Input:}  $d = \overline{\psi}(\mathbb{U})+1$ here $\mathbb{U}$ is the column space of the underlying matrix \\
 \textbf{Initialize:} $\widehat{\mathbf{M}}$ set to $m\times n$ sized null matrix, $\widehat{\mathbf{U}}^0=\emptyset$, $k=0$

\begin{algorithmic}[1]\label{algj}
    \FOR{$i$ from $1$ to $m$,}
    \STATE $\chi^i =\sum_{j=1}^{n}\chi_{ij}$ 
    \ENDFOR
    
    \STATE Sort $\chi^i$s with increasing order and select first $d$ and denote their index set by$-R$         
    \STATE Observe entire $\mathbf{M}_{R:}$ 
    \FOR{$i$ from $1$ to $m$,}
    \STATE $\overline{\chi}^i =\sum_{j\in[n]\setminus R}\chi_{ij}$ 
    \ENDFOR

    \STATE Sort $\overline{\chi}^i$s with increasing order and lets denote $\{i_1,i_2,\ldots,i_n\}$ as $\overline{\chi}^{i_1} \leq \overline{\chi}^{i_2} \leq \ldots \leq  \overline{\chi}^{i_n}$

    \FOR{$h$ from $1$ to $m$,}
    \STATE \hspace{0.2in} \textbf{If} $\| \mathbf{M}_{R:i_h}-\mathcal{P}_{\widehat{\mathbf{U}}^k_{R}} \mathbf{M}_{R :i_h}\|^2 >0$ 
    \STATE \hspace{0.4in} Fully observe $\mathbf{M}_{:i_h}$ add it to the basis                 $\widehat{\mathbf{U}}^k$
    \STATE \hspace{0.4in} Orthogonalize  $\widehat{\mathbf{U}}^k$
    \STATE \hspace{0.4in} $k=k+1$ 
    \STATE \hspace{0.2in} \textbf{Otherwise:} $\widehat{\mathbf{M}}_{:i_h} = \widehat{\mathbf{U}}^k \widehat{\mathbf{U}}^{k^+}_{R:} \widehat{\mathbf{M}}_{R : i_h}$
\ENDFOR

    \STATE return $\widehat{\mathbf{M}}$ 

\end{algorithmic}
 \textbf{Output:} Underlying matrix $\widehat{\mathbf{M}}$
\end{algorithm}

\paragraph{Correctness:} We can see the correctness of $\mathbf{ERHC}$ is due to the correctness of $\mathbf{ERCS}$ as selecting cheapest $\overline{\psi}(\mathbb{U})+1$ is special case of selecting any $\overline{\psi}(\mathbb{U})+1$ many columns and iteration order over the columns doesn't matter similarly for this case too.

\paragraph{Optimality:} Unlike to the previous case, greedy algorithm doesn't give us the cheapest combination of columns and rows. 
Following example provides a matrix and entry costs that shows that greedy algorithm is not optimal.

\[ 
\mathbf{M} = \begin{bmatrix}
    1       & 1 & 2 & 3\\
    1       & 2 & 3 & 4\\
    1       & 3 & 4 & 5\\
    1       & 4 & 5 & 6
\end{bmatrix} 
\hspace{10mm}\mathbf{\chi} = \begin{bmatrix}
    1   & 1 & 4 & 1 \\
    1   & 5 & 3 & 4\\
    4   & 3 & 4 & 4\\
    1   & 4 & 4 & 8
\end{bmatrix} 
\]

The greedy algorithm for this case observes rows $R=\{1,2\}$  and columns $C= \{1,2\}$ which has overall cost of:

\begin{align*}
 (1+1+4+1)+(1+5 &+3+4) +(4+1) +(3+4)   = 32
\end{align*}

However, observing ${R}=\{1,3\}$  and columns ${C}= \{1,3\}$ would give us overall cost of:

\begin{align*}
 (1+1+4+1)+(4+3 +4+4) +(1+1) +(3+4) = 31
\end{align*}

which is cheaper than greedy algorithm.\\\\
However, with the same cost matrix, there are other matrices that shares the same column space as $\mathbf{M}$ (therefore the same column space sparsity number) but greedy algorithm is still optimal.
For the same cost matrix with a slightly modified underlying matrix, we can give an example:

\[ 
\overline{\mathbf{M}} = \begin{bmatrix}
    1       & 1 & 2 & 2\\
    1       & 2 & 2 & 3\\
    1       & 3 & 2 & 4\\
    1       & 4 & 2 & 5
\end{bmatrix} 
\hspace{10mm}\mathbf{\chi} = \begin{bmatrix}
    1   & 1 & 4 & 1 \\
    1   & 5 & 3 & 4\\
    4   & 3 & 4 & 4\\
    1   & 4 & 4 & 8
\end{bmatrix} 
\]

This gives us the conclusion, with just information of the observation cost matrix and column space sparsity number, we cannot pick theoretical optimal set of rows and columns that is guaranteed carrying all of information of the underlying matrix.
\paragraph{2-Optimality:} Even though greedy algorithm cannot return the optimal set of rows and columns, here we show that the overall cost of the cost oof the algorithm is at most twice expensive than optimal. \\[2.5ex]
We denote the row set and column set parameter of optimal 2-stage algorithm $\tilde{R}$ and $\tilde{C}$ and cost of it by $\sigma_{OPT}$.
Then, we can decompose optimal soluton into its parts as following:

    $$\sigma_{OPT}=\chi(\mathbf{M}_{\tilde{R}:}) + \chi(\mathbf{M}_{:\tilde{C}}) - \chi    (\mathbf{M}_{\tilde{R}:\tilde{C}})$$.
    
Trivially,  both of the following inequalities satisfied 

$$\chi(\mathbf{M}_{\tilde{R}:\tilde{C}}) \leq  \chi(\mathbf{M}_{:\tilde{C}})$$

$$\chi(\mathbf{M}_{\tilde{R}:\tilde{C}}) \leq  \chi(\mathbf{M}_{\tilde{R}:})$$

which these inequalities implies that  

$$\sigma_{OPT} \geq \mathrm{max}(\chi(\mathbf{M}_{:\tilde{C}}) , \chi(\mathbf{M}_{\tilde{R}:})).$$

Now lets decompose cost of greedy algorithm to its pieces:

    $$\sigma_{G}=\chi(\mathbf{M}_{R:}) + \chi(\mathbf{M}_{:C}) - \chi(\mathbf{M}_{R:C})$$
    
Note that the greedy algorithm doesn't necessarily selects cheapest basis columns, however selected columns minimizes the overall cost after rows selected. 
Therefore, we conclude that if we denote the set of cheapest columns by $C^B$, then the following inequality satisfied:

\begin{align*}
\sigma_{G} =\chi(\mathbf{M}_{R:}) + \chi(\mathbf{M}_{:C}) - \chi(\mathbf{M}_{R:C})     
& \leq \chi(\mathbf{M}_{R:}) + \chi(\mathbf{M}_{:C^B}) - \chi(\mathbf{M}_{R:C^C})  \\[1.2ex]
& \leq 2 \hspace{1mm} \mathrm{max}\big(\chi(\mathbf{M}_{R:}) , \chi(\mathbf{M}_{:C^B}) \big)
\end{align*}

As we discussed before in order to have guarantee that we will be able to have full information to detect linearly independent columns we need to observe at least $\psi(\mathbb{U})+1$ many rows.\\[2.5ex]
Moreover, as greedy algorithm observe exactly $\overline{\psi}({\mathbb{U}})+1$ many rows by choosing cheapest columns we are guaranteed to have:

$$ \chi(\mathbf{M}_{R:}) \leq  \chi(\mathbf{M}_{\tilde{R}:})$$

Similarly as $C^B$ represents the set of cheapest columns, we have:

$$ \chi(\mathbf{M}_{:C^B}) \leq  \chi(\mathbf{M}_{:\tilde{C}})$$

which together implies 

$$ \mathrm{max}\big(\chi(\mathbf{M}_{R:}), \chi(\mathbf{M}_{:C^B}) \big) \leq  \mathrm{max}\big(\chi(\mathbf{M}_{\tilde{R}:}), \chi(\mathbf{M}_{:\tilde{C}}) \big).$$

Putting all inequalities together we conclude:
\begin{align*}
\sigma_G \leq 2 \hspace{1mm}\mathrm{max}\big(\chi(\mathbf{M}_{R:}) , \chi(\mathbf{M}_{:C^B}) \big)
\leq 2 \hspace{1mm} \mathrm{max}\big(\chi(\mathbf{M}_{\tilde{R}:}) , \chi(\mathbf{M}_{:\tilde{C}}) \big) 
\leq 2 \sigma_{OPT}
\end{align*}

Therefore, we conclude that greedy algorithm gives us 2-optimal algorithm.
\paragraph{Tightness:} In the following example, we see that greedy algorithm cannot guarantee better than 2-optimality:\\

\[ \mathbf{\chi} = \begin{bmatrix}
    \frac{\epsilon}{100}   & \frac{\epsilon}{100} & \frac{\epsilon}{100} & \frac{\epsilon}{100} & 10-\epsilon & 10-\epsilon  \medskip\\
    \frac{\epsilon}{100}   & \frac{\epsilon}{100} & \frac{\epsilon}{100} & \frac{\epsilon}{100} & 10-\epsilon & 10-\epsilon  \medskip\\
    10  & 10 & \frac{\epsilon}{100} & \frac{\epsilon}{100} & \frac{\epsilon}{100} & \frac{\epsilon}{100}  \medskip \\
    10  & 10 & \frac{\epsilon}{100} & \frac{\epsilon}{100} & \frac{\epsilon}{100} & \frac{\epsilon}{100} \medskip \\
    \frac{\epsilon}{100}  & \frac{\epsilon}{100} & 10 & 10 & 10-\epsilon & 10-\epsilon \medskip\\
    \frac{\epsilon}{100}  & \frac{\epsilon}{100} & 10 & 10 & 10-\epsilon & 10-\epsilon \medskip\\
\end{bmatrix} 
\]

It is clear that optimal choice is ${C}=\{1,2\}$ and ${R}=\{3,4\}$ which gives the cost of :

$$\sigma_{OPT} = 10+10+10+10 + 16 \times \frac{\epsilon}{100} =40 + \frac{\epsilon}{6.25}$$
However, greedy algortihm will pick ${R}=\{1,2\}$ in the first stage which has overall cost of 

$$(10-\epsilon)+(10-\epsilon)+(10-\epsilon)+(10-\epsilon) + 8 \times \frac{\epsilon}{100} =40-4\epsilon+   \frac{\epsilon}{12.5}$$

Then in the next stage it choose columns ${C}=\{5,6\}$ which also has cost of 

$$(10-\epsilon)+(10-\epsilon)+(10-\epsilon)+(10-\epsilon) + 4 \times \frac{\epsilon}{100} =40-4\epsilon + \frac{\epsilon}{25}$$

 which all together cumulative cost is 

 $$(40-4\epsilon) +   \frac{\epsilon}{12.5} + (40-4\epsilon) +   \frac{\epsilon}{25}= 80-8\epsilon + \frac{3}{25}\epsilon.$$
 
 To find the fraction of this cost to optimal cost we get
 
 $$\frac{\sigma_G}{\sigma_{OPT}} = \frac{80-8\epsilon + \frac{2\epsilon}{25}}{40+\frac{\epsilon}{6.25}}  \approx 2 -\frac{\epsilon}{5}.$$
 
Therefore for any number smaller than $2$, we can choose an $\epsilon$ which ratio of the cost of greedy algorithm to optimal set is larger than that number.
This implies that, 2-optimality of the algorithm $\mathbf{ERHC}$ is tight.

\bibliography{collas2022_conference}
\bibliographystyle{collas2022_conference}

\appendix

\end{document}

%% file: math_commands.tex
\usepackage{amsmath,amsfonts,bm}

\def\eqref#1{equation~\ref{#1}}

\def\1{\bm{1}}

\DeclareMathAlphabet{\mathsfit}{\encodingdefault}{\sfdefault}{m}{sl}
\SetMathAlphabet{\mathsfit}{bold}{\encodingdefault}{\sfdefault}{bx}{n}

